\documentclass[twoside,11pt]{article}

\usepackage{jmlr2e_1}

\usepackage{amsmath}
\usepackage{times}
\usepackage[ruled,vlined]{algorithm2e}

\newcommand{\tloss}{\widetilde{\loss}}
\newcommand{\tL}{\widetilde{L}}

\newcommand{\tbL}{\widetilde{\mathbf{L}}}
\newcommand{\tbl}{\widetilde{\mathbf{l}}}
\newcommand{\diff}{d}
\newcommand{\psR}{\overline{R}}

\renewcommand{\Pr}{\mathbb{P}}
\newcommand{\dt}{\displaystyle}
\newcommand{\Ind}[1]{\mathbb{I}\left\{{#1}\right\}}

\newcommand{\loss}{l}
\newcommand{\lb}{\lambda}
\newcommand{\qu}{\alpha}
\newcommand{\bqu}{\boldsymbol{\qu}}

\newcommand{\blb}{\boldsymbol{\lb}}
\newcommand{\ubQ}{\mathcal{Q}}

\usepackage{xcolor}
\usepackage{lipsum} 

\newcommand{\bzero}{\mathbf{0}}
\newcommand{\bL}{\mathbf{L}}
\newcommand{\bl}{\mathbf{l}}
\newcommand{\bq}{\mathbf{q}}
\newcommand{\bs}{\mathbf{s}}
\newcommand{\bx}{\mathbf{x}}

\newcommand{\bz}{\mathbf{z}}
\newcommand{\RR}{\mathbb{R}}
\newcommand{\bp}{\mathbf{p}}
\newcommand{\E}{\mathbb{E}}

\newcommand{\ignore}[1]{}

{\hfill\BlackBox\\[2mm]}

\jmlrheading{}{}{}{}{}{}{Eyal Gofer and Guy Gilboa}

\ShortHeadings{Experts with Lower-Bounded Loss Feedback}{Gofer and Gilboa}
\firstpageno{1}

\begin{document}

\title{Experts with Lower-Bounded Loss Feedback: A Unifying Framework}

\author{\name Eyal Gofer \email eyal.gofer@ee.technion.ac.il \\
       \name Guy Gilboa \email guy.gilboa@ee.technion.ac.il \\
       \addr Faculty of Electrical Engineering\\
       Technion – Israel Institute of Technology\\
       Haifa, 32000, Israel}      
\editor{}

\maketitle

\begin{abstract}
The most prominent feedback models for the best expert problem are the full information and bandit models. In this work we consider a simple feedback model that generalizes both, where on every round, in addition to a bandit feedback, the adversary provides a lower bound on the loss of each expert.
Such lower bounds may be obtained in various scenarios, for instance, in stock trading or in assessing errors of certain measurement devices. For this model we prove optimal regret bounds (up to logarithmic factors) for modified versions of Exp3, generalizing algorithms and bounds both for the bandit and the full-information settings. Our second-order unified regret analysis simulates a two-step loss update and highlights three Hessian or Hessian-like expressions, which map to the full-information regret, bandit regret, and a hybrid of both. Our results intersect with those for bandits with graph-structured feedback, in that both settings can accommodate feedback from an arbitrary subset of experts on each round. However, our model also accommodates partial feedback at the single-expert level, by allowing non-trivial lower bounds on each loss.
\end{abstract}


\section{Introduction}\label{sec:intro}
The best expert setting is a classic online learning framework, where a simple game takes place between a learner and an adversary. In this game, there are $N$ available experts (choices, actions), and $T$ rounds of play, or time steps. On each round $1\leq t\leq T$, an online algorithm $A$, the learner, picks a distribution $\bp_{t}$ over the experts and uses it to randomly select an expert $I_t$. Simultaneously, the adversary assigns the losses of the experts for that round, $\bl_t=(l_{1,t},\ldots,l_{N,t})\in \RR^N$, and the learner incurs the loss $l_{I_t,t}$. The aim of the learner is to minimize its \textit{regret}, defined as ${R_{A,T}=L_{A,T}-\min_{j}\{L_{j,T}\}}$, where $L_{A,t}=\sum_{\tau=1}^{t}l_{I_t,\tau}$ is the cumulative loss of $A$ at time $t$ and $\bL_t = \sum_{\tau=1}^{t} \bl_\tau$ is the cumulative loss of the experts at time $t$. Importantly, a small regret should be achieved regardless of the losses chosen by the adversary. The adversary may determine its choices before the game begins (an \textit{oblivious} or \textit{non-adaptive} adversary) or at the time of assignment (an \textit{adaptive} adversary); as has been often observed, however, the meaning of regret as a comparative benchmark is much clearer for oblivious adversaries.

In the \textit{full information} version of the problem, the learner has full knowledge of the past losses of every expert. The most famous learner for this variant is the Hedge algorithm \citep{Vov90,LW94,FreundSc95}. 
\begin{algorithm}\label{alg:hedge}
\SetAlgoLined
\textbf{Parameters:} A learning rate $\eta > 0$ and initial weights $w_{i,1}>0$, $1\leq i \leq N$.\\
For each round $t=1,\dots,T$
\begin{enumerate}
\item 
Define probabilities $p_{i,t}=w_{i,t}/W_t$, where $W_t=\sum_{i=1}^N w_{i,t}$.
\item 
For each expert $i = 1,\dots,N$, let $w_{i,t+1}=w_{i,t}e^{-\eta l_{i,t}}$.
\end{enumerate}
 \caption{Hedge}
\end{algorithm}
For bounded single-period losses, the expected regret of Hedge has an upper bound of the form $O(\sqrt{T\,\log N})$. This type of bound, which depends only on the time horizon of the game, is referred to as a \textit{zeroth-order} bound. A more general so-called \textit{second-order} bound of the form $O(\sqrt{q\,\log N})$ may also be proven given an upper bound $q$ on the relative quadratic variation of the loss sequence, defined as $\sum_{t=1}^T (\max_i\{l_{i,t}\}-\min_i\{l_{i,t}\})^2$ \citep{CMS}. Both bounds are optimal for the expected regret. Other regret bounds, which depend on the cumulative loss of the best expert (first-order bounds) or more refined second-order notions of variation \citep{Hazan10, c+12} have also been shown. Regret bounds that hold with a desired high probability, rather than in expectation, have also been established, for example, a zeroth-order bound of $O(\sqrt{T\,\log(N/\delta)})$ on the regret of Hedge, which holds with probability at least $1-\delta$ (see, e.g., \citealp{BookCBL}).

In contrast to the full-information setting, in the adversarial \textit{multi-armed bandit} (or \textit{bandit}) setting, the learner observes on each round only the loss of the expert it chooses. The Exp3 algorithm \citep{AuerCeFrSc02}, which is an adaptation of Hedge for this setting, obtains a zeroth-order bound of $O(\sqrt{T N\,\log N})$ on the expected regret that is optimal up to logarithmic factors.\footnote{A different algorithm with an optimal regret bound of $O(\sqrt{TN})$ for an oblivious adversary was later presented by \citet{DBLP:journals/jmlr/AudibertB10}.} The simpler version of Exp3 given here as Algorithm \ref{alg:exp3} \citep[Chapter~3]{banditSurvey}, achieves the same bound for non-adaptive adversaries. A variant of Exp3, named Exp3.P, obtains a similar high-probability bound of $O(\sqrt{T N\,\log (N/\delta)})$. 

Second-order bounds for bandits in terms of total variation were shown by \citet{DBLP:journals/jmlr/HazanK11} and by \citet{bubeck2018sparsity} for bounded single-period losses. The latter give a bound on the pseudo-regret (or expected regret for an oblivious adversary) of $O(\sqrt{V\log N}+ N\log^2 T)$, where $V$ is the total variation.
For more details on the theory of bandits, see, e.g., \citet{banditSurvey,lattimore2020bandit,slivkins2019introduction}.

\begin{algorithm}\label{alg:exp3}
\SetAlgoLined
\textbf{Parameters:} $\eta>0$.\\
Let $\bp_1$ be the uniform distribution over $\{1,\hdots,N\}$, and let $\tbL_{0}=\bzero$.\\
For each round $t=1,\dots,T$
\begin{enumerate}
\item 
Draw an action $I_t$ from the probability distribution $\bp_t$.
\item 
For each action $i=1,\dots,N$, compute the estimated loss ${\dt \tloss_{i,t} = \frac{\loss_{i,t}}{p_{i,t}} \Ind{I_t = i} }$ and \\ update the estimated cumulative loss $\tL_{i,t}=\tL_{i,t-1}+\tloss_{i,t}$.
\item 
Compute the new distribution over actions $\bp_{t+1}=\bigl(p_{1,t+1},\hdots,p_{N,t+1}\bigr)$, where
$$p_{i,t+1} = \frac{\exp{\left(- \eta \tL_{i,t}\right)}}{\sum_{k=1}^{N} \exp{\left(- \eta \tL_{k,t}\right)}}\;.$$
\end{enumerate}
\caption{Exp3}
\end{algorithm}

\subsection{A Generalized Model}
We consider a new model for the best expert setting, where the learner receives, in addition to a bandit feedback, a lower bound on the loss of each expert. More specifically, on each round $t$, the adversary assigns the experts both losses $\bl_t \in \RR^N$ and lower bounds on these losses $\blb_t\in \RR^N$, and the learner receives, simultaneously with its decision, the loss of that decision and all the loss lower bounds. This model is an intermediate between the bandit and full-information models, and further generalizes both. 
To retrieve the bandit setting, the adversary may provide trivial lower bounds, such as zero values when the losses are restricted to the range $[0,1]$. To retrieve the full-information setting, the lower bounds may be the actual losses. 

To directly motivate this model, consider a scenario of stock trading. Here the experts are stocks, and the single-period loss of an expert is minus the single-period change in the logarithm of the price of the stock (a loss that may be either positive or negative). Theoretically, any trade may be executed at market price. However, when trading in large volumes or in small stocks, the stock price reacts in a direction that increases the loss to the trader. Thus, real losses are lower bounded by the theoretic losses calculated from market prices.

Another scenario stems from the fact that the variance of a statistical estimator lower bounds its squared error, through the bias-variance decomposition. For example, when a sensor makes several measurements of the same real-valued quantity, the squared error of these measurements is unknown, unless the ground truth value is ascertained, possibly through the costly work of a human expert. However, the empirical variance of these measurements is always known, and may serve as a lower bound to that error. Thus, if several sensor prototypes go through a series of tests in different labs (where on each occasion a sensor makes several independent measurements, yielding variance), we would want to spend our human-expert budget increasingly on the most promising sensor, that is, the one with the lowest cumulative squared error.

Our model allows a general ``soft'' decomposition of the loss of each expert into a known part and a tentative part. Yet, even in a restricted dichotomous regime where each loss is either fully known or completely unknown, there are interesting hybrid scenarios. Specifically, the availability of feedback for experts may vary with time in an unexpected way. This case is handled in our model by an adversary that assigns the losses themselves as lower bounds for experts with available feedback, and trivial lower bounds otherwise. 

\subsection{Summary of Results}\label{subsec:contributions}
We present an algorithm, Exp3.LB, which is an adaptation of Exp3 for our model. In the limits of the full-information and bandit settings, this algorithm is equivalent to Hedge and Exp3, respectively, and thus its analysis captures both algorithms as special cases. Exp3.LB differs from Exp3 in that it updates the estimated cumulative loss of each expert by adding a sum of two elements. The first is the lower bound on the loss, like Hedge does for the true loss. The other is an estimate of the \textit{slack}, or difference between the true loss and its lower bound, like Exp3 does for the loss itself. 

We prove a tight second-order bound on the expected regret of Exp3.LB against non-adaptive adversaries,\footnote{Technically, we bound the pseudo-regret of Exp3.LB, or $\E L_{Exp3.LB,T} - \min_i \E L_{i,t}$.} which has the form $O\left(\sqrt{\ubQ\,\log N}\right)$, where $$\ubQ = \sum_{t=1}^T\left(\max_i\{\lb_{i,t}\}-\min_i\{\lb_{i,t}\}\right)^2 + \sum_{t=1}^T\|\bl_t - \blb_t\|^2\;.$$

In our analysis, the estimated cumulative loss update step is broken into a bandit half-step and a full-information half-step, yielding three second-order quantities with quadratic upper bounds. The first is the relative quadratic variation of the sequence of lower bounds, which in the special full-information case translates into the usual relative quadratic variation, while the two other quantities disappear. The second quantity is the sum of all squared slacks, which in the special bandit scenario translates into the zeroth-order bandit bound while the other two factors disappear. The third is a hybrid of slacks and loss lower bounds, and is non-zero only for scenarios on the continuum between the bandit and full-information cases. It does not change the order of the bound and hence dropped from the above expression for simplicity.

We give expected regret bounds for a variable subset feedback scenario, where on each round $t$, the losses for an adversarially-chosen subset $S_t$ of the experts is revealed to the learner. This scenario may be modeled by the adversary setting $\lb_{i,t} = l_{i,t}$ for every $i\in S_t$ and $\lb_{i,t}=0$ otherwise, where we assume all losses are in $[0,1]$. Applying Exp3.LB to this scenario, we obtain an expected regret bound of $O\left(\sqrt{(T+\sum_t (N-|S_t|))\,\log N}\right)$, which is optimal up to logarithmic factors if the subsets are identical \citep{alon2017nonstochastic}.

We show that other quantities may replace the lower bounds on the losses in our model and algorithm, yielding regret bounds of a similar form. In particular, we may assume that the adversary is providing \textit{upper} bounds $\upsilon_{i,t}$ on the losses. For this scenario we give an algorithm, Exp3.UB, and bound its expected regret. The algorithm and the bound are similar to those of Exp3.LB except that each occurrence of $\lb_{i,t}$ is replaced by a quantity based on $\upsilon_{i,t}$.

Finally, we provide a variant of Exp3.LB with regret bounds that hold with high probability against an oblivious adversary. This algorithm, named Exp3.LB.P, is adapted from Exp3.LB using a biasing method unlike that of Exp3.P. We show second-order regret bounds of the form 
$O(\sqrt{\ubQ}\,\log(N/\delta))$, and given mild conditions, also $O(\sqrt{\ubQ\,\log(N/\delta)})$, where the bounds hold with probability at least $1-\delta$. For the bounded single-period loss scenario, we prove the bound $O(\sqrt{\max\{T,\ubQ\}\,\log(N/\delta)})$, retrieving the zeroth-order bound types of both Exp3.P and Hedge.

\subsection{Related Work}\label{subsec:related}
Several works have considered a scenario where a graph structure describes the feedback flow to the learner \citep{mannor2011bandits,alon2013bandits,alon2017nonstochastic,alon2015online}. Specifically, given a possibly time-dependent graph whose nodes are the experts, choosing one expert reveals the losses of neighboring experts. These works give regret bounds in terms of graph properties, such as the independence number or the size of the maximal acyclic subgraph. 

The difference from our model is twofold. First, in the graph-based model, the set of additional experts providing feedback is a function of the choice made by the learner, while in our model it is not. In this sense, the graph-based model is more general. However, the feedback in our model is ``soft'', rather than binary (available or not), and here the graph-based model is more limited than ours. Both models can handle a scenario where on each round, the losses of a time-dependent set of experts are revealed to the learner, in addition to the loss of the expert it chose. In these cases, which clearly interpolate between the bandit and full-information settings, the regret bounds in the two models have the same form.

The work of \citet{cesa2017bandit} considered a setting where along with the loss of the chosen action, the learner is given an interval containing each loss. Crucially, and in contrast to our work, this interval is given to the learner on each round \textit{before} it makes its choice. The purpose is to allow the learner to take advantage of easy loss sequences, in this case, where potentially only a few experts should be considered on each round. They provide regret bounds that disappear as the interval size shrinks to zero. In summation, their work is thus not truly related to ours.
  
Finally, we comment that contextual bandits, partial monitoring, and combinatorial bandits all have a more distant connection to the topics discussed in this work. More information on these topics may be found in the bandit literature. 

\subsection{Outline}
In Section \ref{sec:notation} we give some useful notation. Section \ref{sec:main} covers our model, the Exp3.LB algorithm, and its expected regret bound. In Section \ref{sec:special feedback} we give corollaries for some special scenarios of interest. Section \ref{sec:variants} covers variants of our model and corresponding algorithms and bounds. Section~\ref{sec:lower} provides lower bounds on the regret in our model. In Section \ref{sec:hp bounds} we give an algorithm and regret bounds that hold with high probability in our model, and in Section \ref{sec:conclusion} we conclude and discuss future directions. The appendix contains some additional claims.

\section{Miscellaneous Notation}\label{sec:notation}
We use bold face for vectors, most often for time series of vectors in $\RR^N$, such as $\bl_1,\ldots, \bl_T$. Their components are written as $\bl_t = (l_{1,t},\ldots,l_{N,t})$. We use $\|\cdot\|$ for the $L_2$ norm, and for $\bx,\mathbf{y} \in\RR^{N}$, $[\bx,\mathbf{y}]=\{a\mathbf{x}+(1-a)\mathbf{y}:0\leq a\leq1\}$ denotes the line segment between $\bx$ and $\mathbf{y}$.
We write $\Delta_N$ for the probability simplex of $N$ elements, $\Delta_N =  \{\bp\in \RR^N : p_i \geq 0\;\forall i=1,\ldots,N,\;\sum_{i=1}^N p_i = 1\}$.
For $\bx \in \RR^N$, $diag(\bx)$ is the diagonal matrix with $\bx$ as its diagonal. The indicator variable of an event $E$ is denoted by $\Ind{E}$. We will often use the specialized notation $\diff(\bx)=\max_i\{x_i\}-\min_i\{x_i\}$ for $\bx \in \RR^N$, and given a sequence of vectors $\bl_1,\ldots,\bl_T\in \RR^N$ we will denote ${q(\bl_{1:T})=\sum_{t=1}^{T}\diff(\bl_t)^{2} = \sum_{t=1}^{T}(\max_i\{l_{i,t}\}-\min_i\{l_{i,t}\})^2}$, namely, the relative quadratic variation of the sequence.  

\section{ Best Experts with Lower Bounds}\label{sec:main}
We define a feedback model where on each round, after choosing an expert, the learner is given the exact loss of that expert and a lower bound on the losses of all the experts. In what follows we will denote by $\lb_{i,t}$ the real-valued lower bound on the loss of expert $i$ at time $t$ and by $s_{i,t} = l_{i,t}-\lb_{i,t}$ the \textit{slack} between the loss and the lower bound of expert $i$ at time $t$. We observe that once the losses and lower bounds for round $t$ become known, we may subtract $\min_i\{\lb_{i,t}\}$ from all of them, without affecting the problem. In particular, the regret of any algorithm is not affected by subtracting a constant $c_t$ from all losses $l_{1,t},\ldots, l_{N,t}$. We will therefore assume WLOG that $\min_i\{\lb_{i,t}\}\geq 0$ for every $t$, and thus $l_{i,t}\geq 0$ for every $i$ and $t$.

We propose a natural variant of Exp3, which we call Exp3.LB, to handle the lower bound information model. 
\begin{algorithm}\label{alg:exp3.lb}
\SetAlgoLined
\textbf{Parameters:} $\eta>0$.\\
Let $\bp_1$ be the uniform distribution over $\{1,\hdots,N\}$, and let $\tbL_{0}=\bzero$.\\
For each round $t=1,\dots,T$
\begin{enumerate}
\item 
Draw an action $I_t$ from the probability distribution $\bp_t$.
\item 
For each action $i=1,\dots,N$, compute the estimated loss $${\dt
  \tloss_{i,t} = \frac{s_{i,t}}{p_{i,t}} \cdot \Ind{I_t = i}+\lb_{i,t}
}$$ and update the estimated cumulative loss
$\tL_{i,t}=\tL_{i,t-1}+\tloss_{i,t}$.
\item 
Compute the new distribution over actions $\bp_{t+1}=\bigl(p_{1,t+1},\hdots,p_{N,t+1}\bigr)$, where
$$p_{i,t+1} = \frac{\exp{\left(- \eta \tL_{i,t}\right)}}{\sum_{k=1}^{N} \exp{\left(- \eta \tL_{k,t}\right)}}\;.$$
\end{enumerate}
\caption{Exp3.LB}
\end{algorithm}

The difference from Exp3 is in the definition of $\tloss_{i,t}$ in step 2, which now incorporates the lower bounds $\lb_{i,t}$. All other elements remain the same. Note that if $\lb_{i,t} = 0$ for every $i$ and $t$ (the pure bandit case), the algorithm becomes Exp3, and if $s_{i,t} = 0$ for every $i$ and $t$ (the full-information case), it becomes Hedge.

We now prove expected regret bounds for Exp3.LB for an oblivious, or non-adaptive, adversary. Namely, we assume that all losses and lower bounds are decided by the adversary before the beginning of the game with the learner. Technically, we will bound the \textit{pseudo-regret} of Exp3.LB, defined as $\psR_{Exp3.LB,T} = \E L_{Exp3.LB,T} - \min_i \E L_{i,t}$. For oblivious adversaries, the notions of expected regret and pseudo-regret coincide. For the rest of this paper we will assume that the adversary is oblivious.

The analysis is adapted from the work of \citet[Theorem~23]{Gofer14} given originally for the full-information case. The main difference lies in simulating a two-part update step, which first performs the ``bandit part'' of the Exp3.LB update step, namely, adding $(s_{i,t}/p_{i,t})\cdot \Ind{I_t = i}$ to $\tloss_{i,t}$, and then the ``full-information part'', namely, adding $\lb_{i,t}$. 

\begin{theorem}\label{thm:Exp3.LB}
Let $\ubQ$ be an upper bound on $$\frac{1}{2} q(\blb_{1:T})+2\sum_{t=1}^T \|\bs_t\|^2+4\sum_{t=1}^T \max_i\{s_{i,t}\}\cdot\diff(\blb_t)\;.$$
Then for any $\eta > 0$ it holds that 
$$\psR_{Exp3.LB,T} \leq \frac{1}{\eta}\log N + \frac{1}{4}\eta \ubQ\;,$$
and in particular for $\eta = \sqrt{4\log N/\ubQ}$, $$\psR_{Exp3.LB,T} \leq \sqrt{\ubQ \log N}\;.$$
\end{theorem}
\begin{proof}
Let $\Phi(\bL) = -(1/\eta)\log\frac{1}{N}\sum_{j=1}^Ne^{-\eta L_j}$ and define
$$\tL_{i,t-\frac{1}{2}} = \tL_{i,t-1} + \frac{s_{i,t}}{p_{i,t}}\cdot \Ind{I_t = i} $$
for $1\leq t\leq T$, $1\leq i \leq N$, recalling that $\tbL_0 = \bzero$. Noting that $\Phi(\bzero) = 0$ we have 
\begin{equation}\label{eq:1:thm:Exp3.LB}
\Phi(\tbL_T) = \sum_{t=1}^T \Phi(\tbL_t) - \Phi(\tbL_{t-\frac{1}{2}}) + \sum_{t=1}^T \Phi(\tbL_{t-\frac{1}{2}}) - \Phi(\tbL_{t-1})\;.
\end{equation}
Now, for every $\bx,\bx' \in \RR^N$ we have by Taylor's expansion that
$$\Phi(\bx')-\Phi(\bx) = \nabla\Phi(\bx)\cdot (\bx'-\bx) + \frac{1}{2}(\bx'-\bx)^\top\nabla^2\Phi(\bz)(\bx'-\bx)\;,$$
where $\bz\in [\bx,\bx']$. Thus, denoting
\begin{align*}
A_t &=(\tbL_t-\tbL_{t-\frac{1}{2}})^\top\nabla^2\Phi(\bz_t)(\tbL_t-\tbL_{t-\frac{1}{2}})\\
B_t &=(\tbL_{t-\frac{1}{2}}-\tbL_{t-1})^\top\nabla^2\Phi(\bz_{t-\frac{1}{2}})(\tbL_{t-\frac{1}{2}}-\tbL_{t-1})\\
C_t &= (\nabla\Phi(\tbL_{t-\frac{1}{2}})-\nabla\Phi(\tbL_{t-1}))\cdot(\tbL_{t}-\tbL_{t-\frac{1}{2}})\;,
\end{align*}
where $\bz_t \in [\tbL_{t-\frac{1}{2}},\tbL_t]$, $\bz_{t-\frac{1}{2}}\in[\tbL_{t-1},\tbL_{t-\frac{1}{2}}]$ for every $t$, we have from \eqref{eq:1:thm:Exp3.LB} that
\begin{align*}
\Phi(\tbL_T) &= \sum_{t=1}^T \nabla\Phi(\tbL_{t-\frac{1}{2}})\cdot (\tbL_t - \tbL_{t-\frac{1}{2}}) + \frac{1}{2}A_t + \nabla\Phi(\tbL_{t-1})\cdot (\tbL_{t-\frac{1}{2}} - \tbL_{t-1}) + \frac{1}{2}B_t\\
&= \sum_{t=1}^T \nabla\Phi(\tbL_{t-1})\cdot (\tbL_t - \tbL_{t-1}) + \frac{1}{2}A_t + \frac{1}{2}B_t + C_t
\end{align*} 
or 
\begin{equation}\label{eq:2:thm:Exp3.LB}
\Phi(\tbL_T) = \sum_{t=1}^T \bp_t\cdot \tbl_t + \frac{1}{2}A_t + \frac{1}{2}B_t + C_t\;.
\end{equation}
We now turn to the quadratic terms $A_t$ and $B_t$. 
It is a well-known fact that for every $\bx,\bz\in\RR^N$, $$\bx^\top\nabla^2\Phi(\bz)\bx = -\eta Var(Y_{\bp,\bx})\;,$$
 where $Y_{\bp,\bx}$ is a random variable that obtains values in $\{x_1,\ldots,x_N\}$. (This claim is stated in the appendix as Lemma \ref{lem:hessian of phi} for completeness; for its proof, see, e.g., \citealp[Lemma 6]{gofer2016lower}). By Popoviciu's inequality (Lemma \ref{lem:Popoviciu-inequality} in the appendix) we therefore have that
\begin{equation}\label{eq:3:thm:Exp3.LB}
 A_t \geq -\frac{\eta}{4}(\max_i\{\lb_{i,t}\}-\min_i\{\lb_{i,t}\})^2\;.
\end{equation} 
For $B_t$, writing $\bq=\nabla \Phi(\bz_{t-\frac{1}{2}})$, we have by Lemma \ref{lem:hessian of phi} that 
\begin{align*}
B_t &= -\eta (\tbL_{t-\frac{1}{2}}-\tbL_{t-1})^\top(diag(\bq)-\bq\bq^{\top})(\tbL_{t-\frac{1}{2}}-\tbL_{t-1})\\
&= -\eta (q_{I_t}-q_{I_t}^2)\cdot \frac{s_{I_t,t}^2}{p_{I_t,t}^2}\;,\\
&\geq  -\eta q_{I_t}\cdot \frac{s_{I_t,t}^2}{p_{I_t,t}^2}\;,
\end{align*}
where we used the fact that only the index $I_t$ in $\tbL_{t-\frac{1}{2}}-\tbL_{t-1}$ may be non-zero. Furthermore, since $\bz_{t-\frac{1}{2}}\in[\tbL_{t-1},\tbL_{t-\frac{1}{2}}]$, we have that $z_{i,t-\frac{1}{2}} = \tL_{i,t-1}$ for $i \neq I_t$ and  $z_{I_tt-\frac{1}{2}} \geq \tL_{I_t,t-1}$. It follows that $q_{I_t} \leq p_{I_t,t}$, and therefore 
\begin{equation}\label{eq:4:thm:Exp3.LB}
B_t \geq -\eta \cdot \frac{s_{I_t,t}^2}{p_{I_t,t}}\;.
\end{equation}
For bounding $C_t$ we again use the fact that $\bp_t = \nabla\Phi(\tbL_{t-1})$ and $\bp_{t+\frac{1}{2}}=\nabla\Phi(\tbL_{t-\frac{1}{2}})$ are probability vectors and that $\tbL_{t-\frac{1}{2}}$ and $\tbL_{t-1}$ may differ only by the index $I_t$, where $\tL_{I_t,t-\frac{1}{2}}\geq\tL_{I_t,t-1}$. As a result, $p_{I_t,t+\frac{1}{2}}-p_{I_t,t}\leq 0$ and for $i\neq I_t$, it holds that $p_{i,t+\frac{1}{2}}-p_{i,t}\geq 0$. Recalling that $\blb_t = \tbL_{t} - \tbL_{t-\frac{1}{2}}$, we have that 
\begin{align*}
C_t &= (\bp_{t+\frac{1}{2}}-\bp_t)\cdot \blb_t \geq \min_i\{\lb_{i,t}\}\cdot\sum_{i\neq I_t} (p_{i,t+\frac{1}{2}}-p_{i,t}) + \max_i\{\lb_{i,t}\}\cdot (p_{I_t,t+\frac{1}{2}}-p_{I_t,t})\\
& = \min_i\{\lb_{i,t}\}\cdot (p_{I_t,t}-p_{I_t,t+\frac{1}{2}}) + \max_i\{\lb_{i,t}\}\cdot (p_{I_t,t+\frac{1}{2}}-p_{I_t,t})\\
& = (p_{I_t,t+\frac{1}{2}}-p_{I_t,t})\cdot \diff(\blb_t)\;.
\end{align*}
Now, by a first-order Taylor expansion of $f(\bx) = \frac{\partial \Phi(\bx)}{\partial x_{I_t}}$ we have for some $\bz' \in [\tbL_{t-1},\tbL_{t-\frac{1}{2}}]$ that
\begin{align*}
p_{I_t,t+\frac{1}{2}}-p_{I_t,t} &= f(\tbL_{t-\frac{1}{2}})- f(\tbL_{t-1})= \nabla f(\bz')\cdot (\tbL_{t-\frac{1}{2}}-\tbL_{t-1}) \\
&= \frac{\partial^2 \Phi(\bz')}{\partial x_{I_t}^2}\cdot\frac{s_{I_t,t}}{p_{I_t,t}}\;.
\end{align*}
Again by Lemma \ref{lem:hessian of phi}, for $\bp' = \nabla \Phi(\bz')$ we have 
$$ \frac{\partial^2 \Phi(\bz')}{\partial x_{I_t}^2} = \eta (p'^2_{I_t}-p'_{I_t}) \geq -\eta p'_{I_t} \;,$$ 
and again, since $\bz' \in [\tbL_{t-1},\tbL_{t-\frac{1}{2}}]$, we have $p'_{I_t} \leq p_{I_t,t}$.
Therefore, 
$$ p_{I_t,t+\frac{1}{2}} - p_{I_t,t} \geq  -\eta p_{I_t,t}\cdot\frac{s_{I_t,t}}{p_{I_t,t}}\;,$$
yielding that 
\begin{equation}\label{eq:5:thm:Exp3.LB}
C_t \geq  -\eta s_{I_t,t}\cdot \diff(\blb_t)\geq -\eta \max_i\{s_{i,t}\}\cdot \diff(\blb_t)\;.
\end{equation}
Finally, we observe that for every $k$,
\begin{align*}
\Phi(\tbL_T)-\tL_{k,T} 
&= -\frac{1}{\eta}\log\frac{1}{N}\sum_{j=1}^N e^{-\eta \tL_{j,T}}+\frac{1}{\eta}\log e^{-\eta \tL_{k,T}}\\
&=\frac{1}{\eta}\log\left(\frac{N\exp(-\eta \tL_{k,T})}{\sum_{j=1}^N e^{-\eta \tL_{j,T}}}\right)\leq  \frac{1}{\eta}\log N\;.
\end{align*} 
Combining this with \eqref{eq:2:thm:Exp3.LB}, \eqref{eq:3:thm:Exp3.LB}, \eqref{eq:4:thm:Exp3.LB}, and \eqref{eq:5:thm:Exp3.LB} and rearranging, we get that for every $k$,
\begin{equation}\label{eq:7:thm:Exp3.LB}
\sum_{t=1}^T \bp_t\cdot \tbl_t - \tL_{k,T}\leq \frac{1}{\eta}\log N +  \frac{\eta}{8}\cdot q(\blb_{1:T}) + \frac{\eta}{2}\sum_{t=1}^T \frac{s_{I_t,t}^2}{p_{I_t,t}} +\eta \sum_{t=1}^T \max_i\{s_{i,t}\}\cdot\diff(\blb_t)\;.
\end{equation}
We can now take expectations on both sides, preserving the inequality. On the r.h.s., we have that
$$\E \left[\sum_{t=1}^T \frac{s_{I_t,t}^2}{p_{I_t,t}}\right]=\sum_{t=1}^T \E\E_{I_t\sim p_t} \frac{s_{I_t,t}^2}{p_{I_t,t}}=\sum_{t=1}^T \sum_{i=1}^N s_{i,t}^2\;,$$
where we used the rule of conditional expectations. On the l.h.s., we have 
\begin{align*}
\E\left[\sum_{t=1}^T \bp_t\cdot \tbl_t\right] &= \E\left[\sum_{t=1}^T \bp_t \cdot \blb_t + l_{I_t,t}- \lb_{I_t,t}\right]\\
&= \E\left[\sum_{t=1}^T \bp_t \cdot \blb_t\right]+\E\left[\sum_{t=1}^T l_{I_t,t}\right]-\E\left[\sum_{t=1}^T \lb_{I_t,t}\right]\\
&= \E\left[\sum_{t=1}^T \bp_t \cdot \blb_t\right]+\E\left[L_{Exp3.LB,T}\right]-\E\left[\sum_{t=1}^T \bp_t \cdot \blb_t\right]\\
&=\E\left[L_{Exp3.LB,T}\right]\;.
\end{align*}
In addition we have that 
\begin{align*}
\E\left[\tL_{k,T}\right] &= \sum_{t=1}^T \E[\tloss_{k,t}]=\sum_{t=1}^T \E\E_{I_t\sim p_t}\left[\lb_{k,t}+\frac{\loss_{k,t}-\lb_{k,t}}{p_{k,t}} \Ind{I_t = k} \right]\\
&= \sum_{t=1}^T \E\left[\lb_{k,t}+(l_{k,t}-\lb_{k,t})\right]\\
&=\E[L_{k,T}]\;.
\end{align*}
Thus, taking expectations in \eqref{eq:7:thm:Exp3.LB} yields that for every $k$,
$$\E\left[L_{Exp3.LB,T}\right]-\E[L_{k,T}] \leq \frac{1}{\eta}\log N+\frac{\eta}{8}\cdot q(\blb_{1:T})+\frac{\eta}{2}\sum_{t=1}^T \sum_{i=1}^N s_{i,t}^2+\eta \sum_{t=1}^T \max_i\{s_{i,t}\}\cdot\diff(\blb_t)\;.$$
We thus have that
$$\psR_{Exp3.LB,T} \leq \frac{1}{\eta}\log N + \frac{1}{4}\eta\ubQ\;,$$
and in particular for $\eta = \sqrt{4\log N/\ubQ}$, 
$$\psR_{Exp3.LB,T} \leq \sqrt{\ubQ \log N}\;,$$
completing the proof.
\end{proof}
The bound of Theorem \ref{thm:Exp3.LB} may be simplified without changing its order up to multiplicative constants. Note that for any $a\geq 0$ and $\bx\in\RR^N$ with non-negative entries, we have 
$$2 \max_i\{x_i\}\cdot a \leq 2 \|\bx\|\cdot a \leq  a^2 + \|\bx\|^2\;,$$
and therefore
$$\frac{1}{2}(a^2 + \|\bx\|^2)\leq\frac{1}{2} a^2 + 2 \|\bx\|^2 + 4 \max_i\{x_i\}\cdot a\leq 4(a^2 + \|\bx\|^2)\;. $$ 
Substituting $\diff(\blb_t)$ for $a$ and $\bs_t$ for $\bx$, and summing over $t$, we get
$$\frac{1}{2}\sum_{t=1}^T(\diff(\blb_t)^2 + \|\bs_t\|^2)\leq\sum_{t=1}^T\frac{1}{2} \diff(\blb_t)^2 + 2 \|\bs_t\|^2 + 4 \max_i\{s_{i,t}\}\cdot \diff(\blb_t)\leq 4\sum_{t=1}^T(\diff(\blb_t)^2 + \|\bs_t\|^2)\;. $$
Recalling that $q(\blb_{1:T}) = \sum_{t=1}^T\diff(\blb_t)^2$ by definition, we get the following.
\begin{corollary}\label{cor:simplified Exp3.LB bound}
If $\ubQ'$ is an upper bound on $4(q(\blb_{1:T}) + \sum_{t=1}^T\|\bs_t\|^2)$, then taking 
$\eta = \sqrt{4\log N/\ubQ'}$, it holds that $\E R_{Exp3.LB,T} \leq \sqrt{\ubQ' \log N}$.
\end{corollary}
Finally, we note that a major chunk of the proof of Theorem \ref{thm:Exp3.LB} holds in a more general scenario. Specifically, even if the algorithm used an arbitrary $0\leq s'_{i,t}$ instead of $s_{i,t}$, Equation \ref{eq:7:thm:Exp3.LB} would still hold.  This fact will be useful when we consider our high-probability variant, and we therefore state the following corollary:
\begin{corollary}\label{cor:for hp alg}
Replacing $s_{i,t}$ with some $s'_{i,t}\geq 0$ in Exp3.LB for every $i$ and $t$, it holds for every $k$ that
$$
\sum_{t=1}^T \bp_t\cdot \tbl_t - \tL_{k,T}\leq \frac{1}{\eta}\log N +  \frac{\eta}{8}\cdot q(\blb_{1:T}) + \frac{\eta}{2}\sum_{t=1}^T \frac{s'^2_{I_t,t}}{p_{I_t,t}} +\eta \sum_{t=1}^T \max_i\{s'_{i,t}\}\cdot\diff(\blb_t)\;.
$$
\end{corollary}

\subsection{Unknown Horizon}
In Theorem \ref{thm:Exp3.LB} we were able to set $\eta$ optimally, assuming foreknowledge of an upper bound $\ubQ$ on the quantity of interest. To remove such an assumption, it is customary to use a `doubling trick', namely, to start the algorithm with a small initial guess for the upper bound, and whenever the guess is exceeded, double it and restart the algorithm. The resulting analysis typically yields a regret bound of the same general order.

In our setting, however, the slack data for unchosen actions is not observable. This is not a problem if all slacks are known to be zero (full information), but it hinders the use of a doubling trick in the most general setting. To overcome this issue to some extent, we may assume that the losses are bounded, WLOG in $[0,1]$, s.t. for every $i$ and $t$ we may replace $s_{i,t}$ in the regret bound with its upper bound $1-\lb_{i,t}$. In this case it is straightforward to show the following.
\begin{corollary}[unknown horizon]\label{cor:unknown horizon}
If $l_{i,t}\in[0,1]$ for every $i$ and $t$, then in conjunction with a doubling trick, the regret of Exp3.LB satisfies $\E R_{Exp3.LB,T}=O\left(\sqrt{\max\{Q_{uh},1\}\log N}\right)$, where 
$Q_{uh} = \frac{1}{2}q(\blb_{1:T})+2\sum_{t=1}^T \sum_{i=1}^N (1-\lb_{i,t})^2+4\sum_{t=1}^T \max_i\{1-\lb_{i,t}\}\cdot\diff(\blb_t)$.
\end{corollary}

\section{Special Feedback Settings}\label{sec:special feedback}
The regret bound of Exp3.LB given in Theorem \ref{thm:Exp3.LB} is dominated by the quantity $\mathcal{Q}$, which comprises three distinct terms. These will be referred to as the \textit{full information term}, $q(\blb_{1:T})$, the \textit{bandit term}, $\sum_{t=1}^T\|\bs_t\|^2$, and the term $\sum_{t=1}^T \max_i\{s_{i,t}\}\cdot\diff(\blb_t)$, which we will call the \textit{hybrid term}. 

In the full information and bandit scenarios, the bound degenerates to the appropriate single term. Specifically, in the full-information case, all slacks $s_{i,t}$ become zero and Exp3.LB becomes the Hedge algorithm. Theorem \ref{thm:Exp3.LB} immediately retrieves a known second-order regret bound for Hedge \citep[Theorem~23]{Gofer14}.
\begin{corollary}[full-information feedback]\label{cor:full information}
If $q$ is an upper bound on $q(\bl_{1:T})$ and $\eta = \sqrt{8\log N / q}$ then $\E R_{Hedge,T} \leq \sqrt{(q/2)\log N}$.
\end{corollary}

In the bandit case, making the standard assumption that $l_{i,t}\in[0,1]$ for every $i$ and $t$, we may take $\lb_{1,t}=\ldots=\lb_{N,t}=0$, s.t. $\diff(\blb_t)=0$ for every $t$ and $q(\blb_{1:T})=0$. We have for every $i$ and $t$ that $s_{i,t}\leq 1$, and Exp3.LB becomes Exp3.  Theorem \ref{thm:Exp3.LB} then yields the zeroth-order bound for Exp3 (see, e.g., \citealp[Theorem~3.1]{banditSurvey}).
\begin{corollary}[bandit feedback]\label{cor:bandit model}
If $\eta = \sqrt{2\log N/(NT)}$, then $\E R_{Exp3,T} \leq \sqrt{2TN\log N}$.
\end{corollary}

There is, however, another interesting setting highlighted by the regret bound of Exp3.LB. When on each round $t$, we have either $\diff(\blb_t)=0$ or $\max_i\{s_{i,t}\}=0$, the hybrid term disappears. This happens in particular if on every round the adversary provides either a bandit feedback or a full-information feedback, possibly by adversarial choice. For this scenario, Theorem \ref{thm:Exp3.LB} gives the following zeroth-order bound.
\begin{corollary}[mixed feedback]\label{cor:mixed model}
If there are $T_{b}$ bandit feedback rounds and $T_{f}$ full-information feedback rounds then for $\eta = \sqrt{8\log N/(T_f+4NT_b)}$ it holds that $$\E R_{Exp3.LB,T} \leq \sqrt{(T_f/2+2NT_b)\log N}\;.$$
\end{corollary}
It should be noted that even without foreknowledge of $T_b$ and $T_f$, a standard doubling trick on $T_f/2+2NT_b$ yields the same order of bound, namely, $O(\sqrt{(T_f/2+2NT_b)\log N})$.

More generally, we may consider a variable subset feedback scenario, where on each round $t$, the true losses for a subset $S_t$ of the experts, chosen by the adversary, are revealed to the learner. For losses bounded in $[0,1]$, we have that 
\begin{align*}
\frac{1}{2} q(\blb_{1:T})+2\sum_{t=1}^T \|\bs_t\|^2+4\sum_{t=1}^T \max_i\{s_{i,t}\}\cdot\diff(\blb_t)
& \leq \frac{T}{2} + 2\sum_{t=1}^T (N-|S_t|) + 4T\\
& \leq \frac{9}{2} T + 2\sum_{t=1}^T(N-|S_t|)\;,
\end{align*}    
and applying Exp3.LB thus yields the following.
\begin{corollary}[variable set feedback]\label{cor:var subsets}
If for every $t$ the learner receives feedback for a subset $S_t$ of the experts, then for $\eta = \sqrt{8\log N/(9T + 4\sum_t(N-|S_t|))}$ it holds that $$\E R_{Exp3.LB,T} \leq \sqrt{(9T/2 + 2\sum_{t=1}^T(N-|S_t|))\log N}\;.$$
\end{corollary}
Again, a doubling trick is applicable even if the sets are not known in advance, yielding the same order of bound, written more succinctly as $O(\sqrt{(T+\sum_t(N-|S_t|))\log N})$.

\section{Model and Algorithm Variants}\label{sec:variants}
In our model, lower bounds on the losses feature as extra information given to the learner, and then play a role in the estimated losses defined by Exp3.LB, which is expected. We might ask if other quantities could feature as the extra information, and how Exp3 should be modified to accommodate them. One natural choice is upper bounds on the losses instead of lower bounds. 

Let $\qu_{i,t}$, for $1\leq i \leq N$, $1\leq t\leq T$, be arbitrary quantities, and assume that the adversary reveals $\qu_{1,t},\ldots,\qu_{N,t}$ to the learner along with the loss of the chosen action $I_t$.  We then define a variant of Exp3, denoted by Exp3.$\qu$, which is the same as Exp3 except that the estimated loss becomes
$$\tloss_{i,t} = \frac{\loss_{i,t}-\qu_{i,t}}{p_{i,t}}\Ind{I_t = i}+\qu_{i,t}\;.$$

A careful examination of the proof of Theorem \ref{thm:Exp3.LB} reveals that $\lb_{i,t}$ may be replaced by $\qu_{i,t}$ (and of course, $s_{i,t}$ by 
$l_{i,t}-\qu_{i,t}$), as long as $l_{i,t} - \qu_{i,t} \geq 0$ for every $i$ and $t$. We thus obtain the following more general form of Theorem \ref{thm:Exp3.LB} (written here more succinctly):
\begin{theorem}\label{thm:Exp3.alpha}
Let $\ubQ_\qu$ be an upper bound on $$\frac{1}{2} q(\mathbf{\bqu}_{1:T})+2\sum_{t=1}^T \sum_{i=1}^N (l_{i,t}-\qu_{i,t})^2+4\sum_{t=1}^T \max_i\{l_{i,t}-\qu_{i,t}\}\cdot\diff(\bqu_t)\;.$$
Then for $\eta = \sqrt{4\log N/\ubQ_\qu}$, it holds that $\E R_{Exp3.\qu,T} \leq \sqrt{\ubQ_\qu \log N}$.
\end{theorem}

To apply this theorem in a case where upper bounds on the losses are provided, we need an extra step. The reason is that we cannot simply take $\qu_{i,t}$ to be the upper bound $\upsilon_{i,t}$ on $l_{i,t}$, since the requirement $l_{i,t} - \qu_{i,t} \geq 0$ would be violated. However, assuming $\mathcal{M}_t$ is an upper bound on $\max_i\{\upsilon_{i,t} - l_{i,t}\}$ known to the learner, then $\qu_{i,t} = \upsilon_{i,t} - \mathcal{M}_t$ would be a valid choice for $\qu_{i,t}$. 
Denoting the resulting algorithm by Exp3.UB, we obtain the following.
\begin{corollary}\label{cor:upper bounds}
For the above scenario,  
if $\ubQ_\qu$ is an upper bound on $$\frac{1}{2} q(\bqu_{1:T})+2\sum_{t=1}^T \sum_{i=1}^N (l_{i,t}-\qu_{i,t})^2+4\sum_{t=1}^T \max_i\{l_{i,t}-\qu_{i,t}\}\cdot\diff(\bqu_t)\;,$$
then for $\eta = \sqrt{4\log N/\ubQ_\qu}$, it holds that $\E R_{Exp3.UB,T} \leq \sqrt{\ubQ_\qu \log N}$.
\end{corollary}

\section{Lower Bounds}\label{sec:lower}
The upper bounds on the expected regret of Exp3.LB that were shown in Section \ref{sec:main} featured quantities of the form $\Theta(q(\blb_{1:T}) + \sum_t\|\bs_t\|^2)$. Given a value $Q = q(\blb_{1:T}) + \sum_t\|\bs_t\|^2$, we may consider either a full-information scenario, where $Q = q(\bl_{1:T})$ or a bandit scenario, where $Q = \sum_t\|\bl_t\|^2$. We may then use existing lower bounds, and in both cases these bounds are of the form $\Omega(\sqrt{Q})$.

We may also examine a more elaborate requirement, where we are prescribed both $Q_1 = q(\blb_{1:T})$ and $Q_2=\sum_t\|\bs_t\|^2$. In this case, since 
$$q(\blb_{1:T}) + \sum_t\|\bs_t\|^2 = \Theta(\max\{q(\blb_{1:T}), \sum_t\|\bs_t\|^2\})\;,$$ we may consider a full-information scenario for $Q_1$ if $Q_1\geq Q_2$, and a bandit scenario for $Q_2$, otherwise. In both cases we can then add artificial rounds to fulfill the rest of the prescription ($Q_2$ or $Q_1$) without possibly decreasing the expected regret. Existing lower bounds then yield an $\Omega(\sqrt{\max\{Q_1,Q_2\}})$ or equivalently, $\Omega(\sqrt{Q_1+Q_2})$, as before.

Our bounds are thus tight up to logarithmic factors for the above requirements.  We comment, however, that in principle, there might be more elaborate requirements that would call for more refined bounds. In this context it is interesting to consider the bound of Corollary \ref{cor:var subsets}, for which results on graph-structured feedback are applicable. For the case where the set $S_t$ is fixed over time, \citet{alon2017nonstochastic} give an optimal lower bound, which is the same as our upper bound up to logarithmic factors.

\section{High-Probability Regret Bounds}\label{sec:hp bounds}
Like Exp3 on which it is based, Exp3.LB uses loss estimates $\tloss_{i,t}$ whose variance may behave like $1/p_{i,t}$. To enable regret bounds that hold with high probability, special care is required to control this variance. The authors of Exp3 introduced the algorithmic variant Exp3.P, which biases the loss estimates and mixes the probability of Exp3 with a suitable uniform distribution. Here we only bias the loss estimate of the chosen action. This allows us to define a bias that depends on the slack, which is unobservable except for the chosen action.\footnote{We note that different alternatives to the biasing mechanism of Exp3.P have also been introduced by \citet{DBLP:journals/jmlr/AudibertB10} and \citet{kocak2014efficient}, the latter in the context of Exp3.} The resulting algorithm, Exp.LB.P, is given below.
\begin{algorithm}\label{alg:exp3.lb.p}
\SetAlgoLined
\textbf{Parameters:} $\eta>0$, $\beta\geq 0$.\\
Let $\bp_1$ be the uniform distribution over $\{1,\hdots,N\}$, and let $\tbL_{0}=\bzero$.\\
For each round $t=1,\dots,T$
\begin{enumerate}
\item 
Draw an action $I_t$ from the probability distribution $\bp_t$.
\item 
Define $x_{i,t} = \frac{\beta s_{i,t}(1- p_{i,t})}{p_{i,t}(1-\beta s_{i,t})+\beta s_{i,t}(1-p_{i,t})}$ and calculate for $i = I_t$.
\item 
For each action $i=1,\dots,N$, compute the estimated loss $${\dt \tloss_{i,t} = \frac{s_{i,t}(1-x_{i,t})\cdot\Ind{I_t = i}}{p_{i,t}}+\lb_{i,t}}$$ and update the estimated cumulative loss $\tL_{i,t}=\tL_{i,t-1}+\tloss_{i,t}$.
\item 
Compute the new distribution over actions $\bp_{t+1}=\bigl(p_{1,t+1},\hdots,p_{N,t+1}\bigr)$, where
$$p_{i,t+1} = \frac{\exp{\left(- \eta \tL_{i,t}\right)}}{\sum_{k=1}^{N} \exp{\left(- \eta \tL_{k,t}\right)}}\;.$$
\end{enumerate}
\caption{Exp3.LB.P}
\end{algorithm}

The new algorithm is identical to Exp3.LB, except for using the corrected slacks $s_{i,t}(1-x_{i,t})$ instead of $s_{i,t}$ in the estimated losses (see step 2 of the algorithm for the definition of $x_{i,t}$). Intuitively, for $\beta>0$, the factor $x_{i,t}$ approaches $1$ for small probabilities and generally prevents extreme behavior of the estimated losses.\footnote{This behavior depends on the magnitude of $\beta s_{i,t}$. It should also be noted that technically, we also allow $\beta s_{i,t}=1$, which implies $x_{i,t}=1$ for all probability values.} For $\beta=0$, Exp3.LB.P simply becomes Exp3.LB. Some useful properties of this correction factor are summarized in the next technical lemma.
\begin{lemma}\label{lem:correction factor properties}
For any $i$ and $t$, if $\beta s_{i,t} \leq 1$, then the correction factor $x_{i,t}$ satisfies the following:
\begin{enumerate}
    \item[(i)]
    $x_{i,t}$ is well-defined, obtains values in $[0,1]$, and $x_{i,t}=0$ iff $\beta s_{i,t}=0$.
    \item[(ii)]
    $x_{i,t} = \beta s_{i,t}\left(\frac{1-x_{i,t}}{p_{i,t}}+2x_{i,t}-1\right)$.
    \item[(iii)]
    $\beta s_{i,t}\left(\frac{1-x_{i,t}}{p_{i,t}}-1\right)\leq 1$.
    \item[(iv)] 
    $p_{i,t} x_{i,t} s_{i,t} \leq \beta s_{i,t}^2\,$.
    \item[(v)] 
    $\frac{\beta s_{i,t}(1-x_{i,t})^2}{p_{i,t}} \leq 1$.
\end{enumerate}
\end{lemma}
\begin{proof}
Derived by simple arithmetic from the definition of $x_{i,t}$.
\begin{enumerate}
    \item[(i)]
Since we always have that $0<p_{i,t}<1$ and since $0 \leq \beta s_{i,t} \leq 1$, the claim is obvious.
    \item[(ii)]
    If $\beta s_{i,t}=0$ then the claim is true, so assuming $\beta s_{i,t}>0$, we have that 
    $$x_{i,t} = \frac{\beta s_{i,t}(1- p_{i,t})}{p_{i,t}+\beta s_{i,t} -2\beta p_{i,t}  s_{i,t}}=\frac{\frac{1}{p_{i,t}}- 1}{\frac{1}{\beta s_{i,t}}+\frac{1}{p_{i,t}} -2}\;,$$
    and therefore,
    $$x_{i,t}\left(\frac{1}{\beta s_{i,t}}+\frac{1}{p_{i,t}} -2\right) = \frac{1}{p_{i,t}}- 1\;.$$
    Rearranging, we get
    $$\frac{x_{i,t}}{\beta s_{i,t}} = \frac{1-x_{i,t}}{p_{i,t}} +2x_{i,t}-1\;,$$
    and multiplying both sides by $\beta s_{i,t}$ yields the claim.
    \item[(iii)]
    From (i) and (ii) we immediately have that 
    $$\beta s_{i,t}\left(\frac{1-x_{i,t}}{p_{i,t}}-1\right) \leq 
    \beta s_{i,t}\left(\frac{1-x_{i,t}}{p_{i,t}}+2x_{i,t}-1\right) = x_{i,t} \leq 1\;.$$
    \item[(iv)]
    From (ii) we have that 
    $$p_{i,t} x_{i,t} s_{i,t}  = \beta s_{i,t}^2(1-x_{i,t}+(2x_{i,t}-1)p_{i,t})\;.$$
    If $z\in [0,1]$, the expression $1-z+(2z-1)p_{i,t}$ attains its maximum for $z=0$ or $z=1$, and therefore
    $$p_{i,t} x_{i,t} s_{i,t} \leq \beta s_{i,t}^2 \max\{1-p_{i,t},p_{i,t}\}\leq \beta s_{i,t}^2\;.$$
    \item[(v)]
    Denote $a=\beta s_{i,t}$. By (ii) we have that
    $$\frac{a(1-x_{i,t})}{p_{i,t}} = {x_{i,t}-a(2x_{i,t}-1)}\;,$$
    and therefore
    \begin{align*}
    \frac{a(1-x_{i,t})^2}{p_{i,t}} &=(1-x_{i,t})(x_{i,t}-a(2x_{i,t}-1))\;.
    \end{align*}
    For $a\in [0,1]$ it is clear that the r.h.s. of the last expression attains its maximum for $a=0$ or $a=1$, yielding that 
    \begin{align*}
    \frac{a(1-x_{i,t})^2}{p_{i,t}} &\leq \max\{x_{i,t}(1-x_{i,t}),(1-x_{i,t})^2\}\leq 1\;.
    \end{align*}
\end{enumerate}
The proof is complete.
\end{proof}
The definition of $x_{i,t}$ is handy in proving the following key lemma, which is modified from Lemma 3.2 in \citet{banditSurvey}. 
\begin{lemma}\label{modified bcb3.2}
Let $\beta>0$ and $\max_{i,t}\{\beta s_{i,t}\} \leq 1$, and fix $1\leq i \leq N$. For every $\delta>0$, it holds with probability at least $1-\delta$ that 
$$\tL_{i,T} \leq L_{i,T} + \frac{1}{\beta}\log\frac{1}{\delta}\;.$$
\end{lemma}
\begin{proof}
Let $\E_t$ be the expectation conditioned on $I_1,\ldots, I_{t-1}$. Since $e^z \leq 1+z+z^2$ for every $z\leq 1$, and using part (iii) of Lemma \ref{lem:correction factor properties}, we have for every $t$ that
\begin{align*}
\E_t\exp\left(\beta s_{i,t} \left(\frac{(1-x_{i,t})\Ind{I_t = i}}{p_{i,t}} - 1\right)\right)
& \leq 1 + \E_t\left[\beta s_{i,t} \left(\frac{(1-x_{i,t})\Ind{I_t = i}}{p_{i,t}} - 1\right)\right] \\
&+ \E_t\left[\beta^2 s^2_{i,t} \left(\frac{(1-x_{i,t})\Ind{I_t = i}}{p_{i,t}} - 1\right)^2\right] \\
& = 1 -\beta x_{i,t}s_{i,t} + \beta^2 s^2_{i,t} \left(\frac{(1-x_{i,t})^2}{p_{i,t}}+2x_{i,t}-1\right)\\
& = 1 +\beta s_{i,t}\left(\beta s_{i,t}\left(\frac{(1-x_{i,t})^2}{p_{i,t}}+2x_{i,t}-1\right)-x_{i,t}\right)\\
& \leq 1 +\beta s_{i,t}\left(\beta s_{i,t}\left(\frac{1-x_{i,t}}{p_{i,t}}+2x_{i,t}-1\right)-x_{i,t}\right)\\
&=1\;,
\end{align*}
where the last equality uses part (ii) of Lemma \ref{lem:correction factor properties}. 
By a further use of induction we obtain that
$$\E\left[\exp\left(\beta \sum_{t=1}^T \frac{s_{i,t}(1-x_{i,t}) \Ind{I_t = i}}{p_{i,t}}-\beta \sum_{t=1}^T  s_{i,t}\right)\right]\leq 1\;.$$
Now, for any random variable $X$, Markov's inequality implies that $\Pr(X > \log(1/\delta)) \leq \delta \E e^X$. Thus, with probability at least $1-\delta$,
$$\beta \sum_{t=1}^T \frac{s_{i,t}(1-x_{i,t}) \Ind{I_t = i}}{p_{i,t}}-\beta \sum_{t=1}^T  s_{i,t} \leq  \log\frac{1}{\delta}\;,$$
or equivalently,
$$ \sum_{t=1}^T \frac{s_{i,t}(1-x_{i,t}) \Ind{I_t = i}}{p_{i,t}}+\lb_{i,t}-l_{i,t} \leq \frac{1}{\beta}\log\frac{1}{\delta} \;,$$
namely, 
$$\tL_{i,T} - L_{i,T} \leq \frac{1}{\beta}\log\frac{1}{\delta}\;,$$
completing the proof.
\end{proof}
We can now bound the regret of Exp3.LB.P.
\begin{theorem}\label{thm:exp3.lb.p}
Let $0< \delta < 1$, let $\ubQ$ be an upper bound on $$\frac{1}{2} q(\blb_{1:T})+2\sum_{t=1}^T \|\bs_t\|^2+4\sum_{t=1}^T \max_i\{s_{i,t}\}\cdot\diff(\blb_t)\;,$$ and set $\eta = \sqrt{\frac{4}{\ubQ}\log N}$. 
\begin{enumerate}
\item[(i)] 
If $\beta = \sqrt{\frac{2}{\ubQ}\log\frac{N+3}{\delta}}$, then assuming $\beta\max_{i,t}\{ s_{i,t}\} \leq 1$, it holds w.p. at least $1-\delta$ that
$R_{Exp3.LB.P,T} = O\left(\sqrt{\ubQ\log\frac{N}{\delta}}\right)$.
\item[(ii)]
If $\beta = \sqrt{\frac{2}{\ubQ}}$, then w.p. at least $1-\delta$, $R_{Exp3.LB.P,T} = O\left(\sqrt{\ubQ}\cdot \log(N/\delta)\right)$.
\item[(iii)]
For the scenario where $l_{i,t}\in [0,1]$ for every $i$ and $t$, requiring $\ubQ \geq 2T$ and setting 
$$\beta = \min\left\{1,\sqrt{\frac{2}{\ubQ}\log\frac{N+3}{\delta}}\right\}$$ 
yields that $R_{Exp3.LB.P,T} = O\left(\sqrt{\ubQ\log(N/\delta)}\right)$ w.p. at least $1-\delta$. This bound implies the zeroth-order regret bound of Exp3.P, $O(\sqrt{NT\log(N/\delta)})$, for the bandit setting and of Hedge, $O(\sqrt{T\log(N/\delta)})$, for the full-information setting.
\end{enumerate}
\end{theorem}
\begin{proof}
Much of the analysis of Exp3.LB is also applicable to Exp3.LB.P, and this shared part is given in Corollary \ref{cor:for hp alg}. Thus, since $s'_{i,t} = s_{i,t}(1-x_{i,t})\in [0,s_{i,t}]$ for every $i$ and $t$, we have for every $k$ that 
\begin{equation}\label{eq:6:thm:Exp3.LB.P}
\sum_{t=1}^T \bp_t\cdot \tbl_t - \tL_{k,T}\leq \frac{1}{\eta}\log N +  \frac{\eta}{8}\cdot q(\blb_{1:T}) + \frac{\eta}{2}\sum_{t=1}^T \frac{s'^2_{I_t,t}}{p_{I_t,t}} +\eta \sum_{t=1}^T \max_i\{s_{i,t}\}\cdot\diff(\blb_t)\;,
\end{equation}
where we additionally replaced $\max_i\{s'_{i,t}\}$ with $\max_i\{s_{i,t}\}$.

We next establish some high-probability bounds. First, by the Azuma-Hoeffding inequality (see, e.g., Lemma A.7 in \citealp{BookCBL}) it holds w.p. at least $1-\delta$ that
\begin{equation}\label{eq:7:thm:Exp3.LB.P}
\sum_{t=1}^T \bp_t\cdot\blb_t - \lb_{I_t,t} \geq -\left(\frac{1}{2}\log\frac{1}{\delta}\sum_{t=1}^T \diff(\blb_t)^2\right)^\frac{1}{2} =  -\left(\frac{1}{2}q(\blb_{1:T})\log\frac{1}{\delta} \right)^\frac{1}{2}\;.
\end{equation}
Then, we have for every $t$ that $0 \leq x_{I_t,t}s_{I_t,t}\leq s_{I_t,t}\leq \max_i\{s_{i,t}\}$ and 
\begin{align*}
\E_t x_{I_t,t}s_{I_t,t}&=\sum_{i=1}^N p_{i,t}x_{i,t}s_{i,t}\leq \sum_{i=1}^N \beta s_{i,t}^2 =  \beta\|\bs_t\|^2\;,
\end{align*} 
where the inequality is by part (iv) of Lemma \ref{lem:correction factor properties}. Thus, again by the Azuma-Hoeffding inequality, it holds w.p. at least $1-\delta$ that  \begin{equation}\label{eq:8:thm:Exp3.LB.P}
\sum_{t=1}^T x_{I_t,t}s_{I_t,t} - \beta\sum_{t=1}^T\|\bs_t\|^2 \leq \left(\frac{1}{2}\log\frac{1}{\delta}\sum_{t=1}^T \max_i\{s_{i,t}^2\}\right)^\frac{1}{2}\;.
\end{equation}
Using \eqref{eq:7:thm:Exp3.LB.P} and \eqref{eq:8:thm:Exp3.LB.P} we obtain that 
\begin{align}\label{eq:9:thm:Exp3.LB.P}
\sum_{t=1}^T \bp_t\cdot\tbl_t &= \sum_{t=1}^T s_{I_t,t}(1-x_{I_t,t})+\bp_t\cdot\blb_t = \sum_{t=1}^T l_{I_t,t} +\sum_{t=1}^T (\bp_t\cdot\blb_t - \lb_{I_t,t})-\sum_{t=1}^T x_{I_t,t}s_{I_t,t} \nonumber \\ 
&\geq 
\sum_{t=1}^T l_{I_t,t} 
-\left(\frac{1}{2}q(\blb_{1:T})\log\frac{1}{\delta} \right)^\frac{1}{2}-\beta\sum_{t=1}^T\|\bs_t\|^2-\left(\frac{1}{2}\log\frac{1}{\delta}\sum_{t=1}^T \max_i\{s_{i,t}^2\}\right)^\frac{1}{2}\;.
\end{align}
Next, we bound the term $\sum_{t=1}^T \frac{s'^2_{I_t,t}}{p_{I_t,t}}$ in a similar way. It holds that
$\E_t \frac{s'^2_{I_t,t}}{p_{I_t,t}}\leq\|\bs_t\|^2$, and in addition,
\begin{align*}
\frac{s'^2_{I_t,t}}{p_{I_t,t}}  &= \frac{s^2_{I_t,t}(1-x_{I_t,t})^2}{p_{I_t,t}}\leq \beta^{-1}s_{I_t,t}\leq \beta^{-1}\max_i\{s_{i,t}\}\;,
\end{align*}
where the first inequality is by part (v) of Lemma \ref{lem:correction factor properties}. We thus have w.p. at least $1-\delta$ that 
\begin{equation}\label{eq:10:thm:Exp3.LB.P}
\sum_{t=1}^T \frac{s'^2_{I_t,t}}{p_{I_t,t}} - \sum_{t=1}^T \|\bs_t\|^2 \leq \left(\frac{1}{2}\beta^{-2}\log\frac{1}{\delta}\sum_{t=1}^T \max_i\{s_{i,t}^2\}\right)^\frac{1}{2}\;.
\end{equation}
Finally, by Lemma \ref{modified bcb3.2} we have w.h.p. that $\tL_{k,T} \leq L_{k,T} +\frac{1}{\beta}\log\frac{1}{\delta}$. 
We may combine this bound with Equation \eqref{eq:9:thm:Exp3.LB.P} to yield
\begin{align}\label{eq:11:thm:Exp3.LB.P}
\sum_{t=1}^T \bp_t\cdot \tbl_t - \tL_{k,T}&\geq \sum_{t=1}^T l_{I_t,t}
-\left(\frac{1}{2}q(\blb_{1:T})\log\frac{1}{\delta} \right)^\frac{1}{2}-\beta\sum_{t=1}^T\|\bs_t\|^2-\left(\frac{1}{2}\log\frac{1}{\delta}\sum_{t=1}^T \max_i\{s_{i,t}^2\}\right)^\frac{1}{2}\nonumber\\&-L_{k,T} -\frac{1}{\beta}\log\frac{1}{\delta}
\;.
\end{align}
From \eqref{eq:6:thm:Exp3.LB.P} and \eqref{eq:10:thm:Exp3.LB.P} we also have that 
\begin{align}\label{eq:12:thm:Exp3.LB.P}
\sum_{t=1}^T \bp_t\cdot \tbl_t - \tL_{k,T}&\leq\frac{1}{\eta}\log N + \frac{\eta}{8}\cdot q(\blb_{1:T}) + \frac{\eta}{2} \sum_{t=1}^T \|\bs_t\|^2 + \frac{\eta}{2\beta}\left(\frac{1}{2}\log\frac{1}{\delta}\sum_{t=1}^T \max_i\{s_{i,t}^2\}\right)^\frac{1}{2} \nonumber\\&+\eta \sum_{t=1}^T \max_i\{s_{i,t}\}\cdot\diff(\blb_t)\;.
\end{align}
Combining Equations \eqref{eq:11:thm:Exp3.LB.P} and \eqref{eq:12:thm:Exp3.LB.P} and rearranging, we obtain that for every $k$, 
\begin{align}\label{eq:13:thm:Exp3.LB.P}
\sum_{t=1}^T l_{I_t,t} - L_{k,T} &\leq\
\frac{1}{\eta}\log N + \frac{\eta}{8}\cdot q(\blb_{1:T}) + \frac{\eta}{2} \sum_{t=1}^T \|\bs_t\|^2 +\eta \sum_{t=1}^T \max_i\{s_{i,t}\}\cdot\diff(\blb_t) \nonumber\\&+\beta\sum_{t=1}^T \|\bs_t\|^2+ \frac{1}{\beta}\log\frac{1}{\delta} + \left(1+\frac{\eta}{2\beta}\right)\left(\frac{1}{2}\log\frac{1}{\delta}\sum_{t=1}^T \max_i\{s_{i,t}^2\}\right)^\frac{1}{2}\nonumber\\&+\left(\frac{1}{2}q(\blb_{1:T})\log\frac{1}{\delta} \right)^\frac{1}{2}\;.
\end{align}
We briefly comment that throughout the proof, a total of $N+3$ events occur w.p. at least $1-\delta$. As usual, we may insure that all of them occur simultaneously w.p. at least $1-\delta$ by using $\delta' = \delta/(N+3)$ instead of $\delta$ and invoking the union bound. 

Now, the first line of the r.h.s. of \eqref{eq:13:thm:Exp3.LB.P} is exactly the regret bound of Theorem~\ref{thm:Exp3.LB} and is minimized similarly. The first line thus becomes simply $\sqrt{\ubQ \log N}$. One may also observe that $q(\blb_{1:T}) \leq 2\ubQ$ and $$\sum_{t=1}^T\max_i\{s_{i,t}^2\}\leq \sum_{t=1}^T\|\bs_t\|^2\leq \frac{1}{2}\ubQ\;.$$ Therefore, setting $\beta = \sqrt{\frac{2}{\ubQ}\log\frac{1}{\delta'}}$, we have that
$$\beta\sum_{t=1}^T \|\bs_t\|^2+ \frac{1}{\beta}\log\frac{1}{\delta'} \leq \sqrt{2\ubQ\log\frac{1}{\delta'}}\;.$$
Writing $A$ for Exp3.LB.P, we thus obtain that 
\begin{align}
R_{A,T} &\leq\
\sqrt{\ubQ \log N} + 
\sqrt{2\ubQ\log\frac{1}{\delta'}} +
\left(1+\sqrt{\frac{\log N}{2\log\frac{1}{\delta'}}}\right)\cdot\sqrt{\frac{1}{4}\ubQ\log\frac{1}{\delta'}}
+\sqrt{\ubQ\log\frac{1}{\delta'}}\nonumber\\
&\leq\left(1+\frac{1}{2\sqrt{2}}\right)\cdot\sqrt{\ubQ \log N} + 
\left(\sqrt{2}+\frac{3}{2}\right)\cdot\sqrt{\ubQ\log\frac{1}{\delta'}}\nonumber\\
&= O\left(\sqrt{\ubQ\log\frac{N}{\delta}}\right)\;,
\end{align}
proving part (i).

To avoid the extra assumption that $\beta\max_{i,t}\{s_{i,t}\}\leq 1$, we can set $\beta = \sqrt{\frac{2}{\ubQ}}$. We thus have that if $\max_{i,t}\{s_{i,t}\}=0$ then $\beta\max_{i,t}\{s_{i,t}\}\leq 1$ trivially, and otherwise,
$$\beta \leq \sqrt{\frac{2}{2\sum_t\|\bs_t\|^2}}\leq \sqrt{\frac{1}{\max_{i,t}\{s_{i,t}\}^2}}= \frac{1}{\max_{i,t}\{s_{i,t}\}}\;,$$
as needed. It now holds that 
$$\beta\sum_{t=1}^T \|\bs_t\|^2+ \frac{1}{\beta}\log\frac{1}{\delta'} \leq \sqrt{\frac{1}{2}\ubQ}\cdot\log\frac{e}{\delta'}\;.$$
We then have from \eqref{eq:13:thm:Exp3.LB.P} that
\begin{align}\label{eq:14:thm:Exp3.LB.P}
R_{A,T} &\leq \sqrt{\ubQ\log N}+\sqrt{\frac{1}{2}\ubQ}\cdot\log\frac{e}{\delta'} + \left(1+\sqrt{\frac{1}{2}\log N}\right)\cdot\sqrt{\frac{1}{4}\ubQ\log\frac{1}{\delta'}}
+\sqrt{\ubQ\log\frac{1}{\delta'}}\nonumber\\
&= O\left(\sqrt{\ubQ}\cdot \log\frac{N}{\delta}\right)\;,
\end{align}
yielding part (ii).

For part (iii), we first comment that if $T$ is known, we may always assume that $\ubQ \geq 2T$ (otherwise we use $\max\{\ubQ,2T\}$ instead of $\ubQ$). Next, we note that the condition $\beta \max_{i,t}\{s_{i,t}\}\leq 1$ is satisfied if $\beta \leq 1$, and in particular by setting $\beta = \min\left\{1,\sqrt{(2/\ubQ)\log(1/\delta')}\right\}$. If $\beta < 1$, we have by part (i) that
$R_{A,T} = O\left(\sqrt{\ubQ\log(N/\delta)}\right)$.
Otherwise, $\log(1/\delta')\geq T$, and it follows trivially that 
$R_{A,T} \leq T \leq \sqrt{Q\log(1/\delta')}$. Therefore, in any case it holds w.p. at least $1-\delta$ that 
\begin{align}\label{eq:15:thm:Exp3.LB.P}
R_{A,T} &= O\left(\sqrt{\ubQ\log(N/\delta)}\right)\;.
\end{align}
It is easy to observe that in the bandit case we may use $\ubQ=2NT$, yielding a regret bound of $O\left(\sqrt{NT\log(N/\delta)}\right)$, and in the full-information case we may use $\ubQ=\max\{T/2, 2T\}=2T$, yielding a regret bound of $O\left(\sqrt{T\log(N/\delta)}\right)$, as required.
\end{proof}
\begin{remark}
The assumption that $\beta\max_{i,t}\{ s_{i,t}\} \leq 1$, which was made in part (i) of Theorem~\ref{thm:exp3.lb.p}, is mild. It holds trivially in the full information case, namely, $\max_{i,t}\{ s_{i,t}\}=0$, and otherwise we have that
$$\beta = \sqrt{\frac{2}{\ubQ}\log\frac{N+3}{\delta}}\leq \sqrt{\frac{\log\frac{N+3}{\delta}}{\sum_{t=1}^T \|\bs_t\|^2}}\;.$$
Thus, it holds if $\max_{i,t}\{s_{i,t}\} \ll \sqrt{\sum_{t=1}^T \|\bs_t\|^2}$, namely, if $\{s_{i,t}\}$ is not concentrated on very few indices.
\end{remark}

Finally, we point out that the expression bounded by $\ubQ$ is in fact of the simpler and more interpretable form $\Theta(q(\blb_{1:T})+\sum_{t=1}^T \|\bs_t\|^2)$. This argument has already been made in the context of Theorem~\ref{thm:Exp3.LB} and formalized in Corollary~\ref{cor:simplified Exp3.LB bound}.

\section{Conclusion}\label{sec:conclusion}
In this work we presented an online learning model that unifies and generalizes the full-information and bandit settings. We gave algorithms and analysis for this model, thus providing a single, generalized, framework. We proved regret bounds that are optimal up to logarithmic factors and handled both the expected regret and the high-probability regret regimes.

Our generalization works by modeling partial knowledge of losses as full knowledge of their lower bounds. This is in contrast to works on graph-structured feedback, where partial knowledge is modeled as full knowledge of losses for subsets of experts. In future work it would be interesting to examine a combination of our model with graph-structured feedback. 

On a more technical aspect, it appears that current methods for proving regret lower bounds are not straightforward to apply for scenarios with slightly elaborate constraints on the losses, including in our model. Lower bounds for such scenarios would either strengthen tightness results or help suggest more refined regret upper bounds.

\acks{ 
We acknowledge support by the Israel Science Foundation (Grant No. 534/19), the Ministry of Science and Technology (Grant 3-15621) and by the Ollendorff Minerva Center.
}



\appendix

\section*{Appendix A. Additional Claims}\label{sec:additional claims}

\begin{lemma}\label{lem:hessian of phi}
Let $\bz\in\RR^N$, $\bp_0\in \Delta_N$, define $\Phi(\bz) = -(1/\eta)\log\sum_{j=1}^Np_{j,0}e^{-\eta z_j}$, and denote $\bp=\nabla\Phi(\bz)$. Then $\nabla^{2}\Phi(\bz)=\eta\cdot(\bp\bp^{\top}-diag(\bp))\preceq 0 $. Moreover, for every $\bx\in\RR^N$, it holds that $\bx^\top\nabla^{2}\Phi(\bz)\bx = -\eta Var(Y_{\bp,\bx})$, where $Y_{\bp,\bx}$ is a random variable that satisfies for every $1\leq i\leq N$ that $\Pr(Y_{\bp,\bx}=x_{i})=\sum_{\{j:x_{j}=x_{i}\}}p_{j}$.
\end{lemma}

\begin{lemma}\textup{(Popoviciu's inequality)}\label{lem:Popoviciu-inequality}
If $X$ is a bounded random variable with values
in $[m,M]$, then $Var(X)\leq(M-m)^{2}/4$, with equality
iff $\Pr(X=M)=\Pr(X=m)=1/2$.
\end{lemma}


\vskip 0.2in

\begin{thebibliography}{22}
\providecommand{\natexlab}[1]{#1}
\providecommand{\url}[1]{\texttt{#1}}
\expandafter\ifx\csname urlstyle\endcsname\relax
  \providecommand{\doi}[1]{doi: #1}\else
  \providecommand{\doi}{doi: \begingroup \urlstyle{rm}\Url}\fi

\bibitem[Alon et~al.(2013)Alon, Cesa-Bianchi, Gentile, and
  Mansour]{alon2013bandits}
Noga Alon, Nicol{\`{o}} Cesa-Bianchi, Claudio Gentile, and Yishay Mansour.
\newblock From bandits to experts: A tale of domination and independence.
\newblock In \emph{Advances in Neural Information Processing Systems}, pages
  1610--1618, 2013.

\bibitem[Alon et~al.(2015)Alon, Cesa-Bianchi, Dekel, and Koren]{alon2015online}
Noga Alon, Nicol{\`{o}} Cesa-Bianchi, Ofer Dekel, and Tomer Koren.
\newblock Online learning with feedback graphs: Beyond bandits.
\newblock In \emph{Annual Conference on Learning Theory}, volume~40. Microtome
  Publishing, 2015.

\bibitem[Alon et~al.(2017)Alon, Cesa-Bianchi, Gentile, Mannor, Mansour, and
  Shamir]{alon2017nonstochastic}
Noga Alon, Nicol{\`{o}} Cesa-Bianchi, Claudio Gentile, Shie Mannor, Yishay
  Mansour, and Ohad Shamir.
\newblock Nonstochastic multi-armed bandits with graph-structured feedback.
\newblock \emph{SIAM Journal on Computing}, 46\penalty0 (6):\penalty0
  1785--1826, 2017.

\bibitem[Audibert and Bubeck(2010)]{DBLP:journals/jmlr/AudibertB10}
Jean-Yves Audibert and S{\'e}bastien Bubeck.
\newblock Regret bounds and minimax policies under partial monitoring.
\newblock \emph{Journal of Machine Learning Research}, 11:\penalty0 2785--2836,
  2010.

\bibitem[Auer et~al.(2002)Auer, Cesa-Bianchi, Freund, and
  Schapire]{AuerCeFrSc02}
Peter Auer, Nicol\`o Cesa-Bianchi, Yoav Freund, and Robert~E. Schapire.
\newblock The nonstochastic multiarmed bandit problem.
\newblock \emph{SIAM Journal on Computing}, 32\penalty0 (1):\penalty0 48--77,
  2002.

\bibitem[Bubeck and Cesa-Bianchi(2012)]{banditSurvey}
S{\'e}bastien Bubeck and Nicol{\`o} Cesa-Bianchi.
\newblock Regret analysis of stochastic and nonstochastic multi-armed bandit
  problems.
\newblock \emph{Foundations and Trends in Machine Learning}, 5\penalty0
  (1):\penalty0 1--122, 2012.

\bibitem[Bubeck et~al.(2018)Bubeck, Cohen, and Li]{bubeck2018sparsity}
S{\'e}bastien Bubeck, Michael Cohen, and Yuanzhi Li.
\newblock Sparsity, variance and curvature in multi-armed bandits.
\newblock In \emph{Algorithmic Learning Theory}, pages 111--127. PMLR, 2018.

\bibitem[Cesa-Bianchi and Lugosi(2006)]{BookCBL}
Nicol{\`{o}} Cesa-Bianchi and G{\'{a}}bor Lugosi.
\newblock \emph{Prediction, Learning, and Games}.
\newblock Cambridge University Press, 2006.

\bibitem[Cesa-Bianchi and Shamir(2017)]{cesa2017bandit}
Nicol{\`o} Cesa-Bianchi and Ohad Shamir.
\newblock Bandit regret scaling with the effective loss range.
\newblock \emph{arXiv preprint arXiv:1705.05091}, 2017.

\bibitem[Cesa-Bianchi et~al.(2007)Cesa-Bianchi, Mansour, and Stoltz]{CMS}
Nicol{\`{o}} Cesa-Bianchi, Yishay Mansour, and Gilles Stoltz.
\newblock Improved second-order bounds for prediction with expert advice.
\newblock \emph{Machine Learning}, 66\penalty0 (2-3):\penalty0 321--352, 2007.

\bibitem[Chiang et~al.(2012)Chiang, Yang, Lee, Mahdavi, Lu, Jin, and Zhu]{c+12}
Chao-Kai Chiang, Tianbao Yang, Chia-Jung Lee, Mehrdad Mahdavi, Chi-Jen Lu, Rong
  Jin, and Shenghuo Zhu.
\newblock Online optimization with gradual variations.
\newblock \emph{Journal of Machine Learning Research - Proceedings Track},
  23:\penalty0 6.1--6.20, 2012.

\bibitem[Freund and Schapire(1997)]{FreundSc95}
Yoav Freund and Robert~E Schapire.
\newblock A decision-theoretic generalization of on-line learning and an
  application to boosting.
\newblock \emph{Journal of Computer and System Sciences}, 55\penalty0
  (1):\penalty0 119--139, 1997.

\bibitem[Gofer(2014)]{Gofer14}
Eyal Gofer.
\newblock Higher-order regret bounds with switching costs.
\newblock \emph{Journal of Machine Learning Research - Proceedings Track},
  35:\penalty0 210--243, 2014.

\bibitem[Gofer and Mansour(2016)]{gofer2016lower}
Eyal Gofer and Yishay Mansour.
\newblock Lower bounds on individual sequence regret.
\newblock \emph{Machine Learning}, 103\penalty0 (1):\penalty0 1--26, 2016.

\bibitem[Hazan and Kale(2010)]{Hazan10}
Elad Hazan and Satyen Kale.
\newblock Extracting certainty from uncertainty: regret bounded by variation in
  costs.
\newblock \emph{Machine Learning}, 80\penalty0 (2-3):\penalty0 165--188, 2010.

\bibitem[Hazan and Kale(2011)]{DBLP:journals/jmlr/HazanK11}
Elad Hazan and Satyen Kale.
\newblock Better algorithms for benign bandits.
\newblock \emph{Journal of Machine Learning Research}, 12:\penalty0 1287--1311,
  2011.

\bibitem[Koc{\'a}k et~al.(2014)Koc{\'a}k, Neu, Valko, and
  Munos]{kocak2014efficient}
Tom{\'a}{\v{s}} Koc{\'a}k, Gergely Neu, Michal Valko, and R{\'e}mi Munos.
\newblock Efficient learning by implicit exploration in bandit problems with
  side observations.
\newblock In \emph{Advances in Neural Information Processing Systems}, pages
  613--621, 2014.

\bibitem[Lattimore and Szepesv{\'a}ri(2020)]{lattimore2020bandit}
Tor Lattimore and Csaba Szepesv{\'a}ri.
\newblock \emph{Bandit algorithms}.
\newblock Cambridge University Press, 2020.

\bibitem[Littlestone and Warmuth(1994)]{LW94}
Nick Littlestone and Manfred~K Warmuth.
\newblock The weighted majority algorithm.
\newblock \emph{Information and Computation}, 108:\penalty0 212--261, 1994.

\bibitem[Mannor and Shamir(2011)]{mannor2011bandits}
Shie Mannor and Ohad Shamir.
\newblock From bandits to experts: On the value of side-observations.
\newblock In \emph{Advances in Neural Information Processing Systems}, pages
  684--692, 2011.

\bibitem[Slivkins(2019)]{slivkins2019introduction}
Aleksandrs Slivkins.
\newblock Introduction to multi-armed bandits.
\newblock \emph{Foundations and Trends in Machine Learning}, 12\penalty0
  (1-2):\penalty0 1--286, 2019.

\bibitem[Vovk(1990)]{Vov90}
Vladimir Vovk.
\newblock Aggregating strategies.
\newblock In \emph{Proceedings of the 3rd Annual Workshop on Computational
  Learning Theory}, pages 371--383, 1990.

\end{thebibliography}
\end{document}